\documentclass[10pt,twocolumn]{article}
\usepackage[utf8]{inputenc}
\usepackage[margin=0.75in]{geometry}
\usepackage{amsmath}
\usepackage{amssymb}
\usepackage{amsthm}
\usepackage{graphicx}
\usepackage{booktabs}
\usepackage{algorithm}
\usepackage{algorithmic}
\usepackage{multirow}
\usepackage[colorlinks=true, allcolors=blue]{hyperref}
\usepackage[backend=bibtex,style=numeric,sorting=none]{biblatex}
\usepackage{tikz}
\usetikzlibrary{positioning, shapes.geometric, arrows.meta, calc, fit, backgrounds}
\usepackage{pgfplots}
\pgfplotsset{compat=1.17}
\usepackage{subcaption}
\usepackage{caption}
\usepackage{float}
\usepackage{placeins}
\usepackage{adjustbox}
\usepackage{pifont}
\usepackage{xcolor}
\usepackage{enumitem}
\setlist{nosep}

\newtheorem{theorem}{Theorem}
\newtheorem{proposition}{Proposition}
\newtheorem{assumption}{Assumption}

\newcommand{\cmark}{\ding{51}}
\newcommand{\xmark}{\ding{55}}

\title{\textbf{JANUS: Structured Bidirectional Generation for Guaranteed Constraints and Analytical Uncertainty}}

\author{Taha Racicot\\
Université Laval\\
\texttt{taha.racicot.1@ulaval.ca}}
\date{}

\begin{document}

\maketitle

\begin{abstract}
High-stakes synthetic data generation faces a fundamental \textbf{Quadrilemma}: achieving \emph{Fidelity} to the original distribution, \emph{Control} over complex logical constraints, \emph{Reliability} in uncertainty estimation, and \emph{Efficiency} in computational cost---simultaneously. State-of-the-art Deep Generative Models (CTGAN, TabDDPM) excel at fidelity but rely on inefficient rejection sampling for continuous range constraints. Conversely, Structural Causal Models offer logical control but struggle with high-dimensional fidelity and complex noise inversion. We introduce \textbf{JANUS} (Joint Ancestral Network for Uncertainty and Synthesis), a framework that unifies these capabilities using a DAG of Bayesian Decision Trees. Our key innovation is \textbf{Reverse-Topological Back-filling}, an algorithm that propagates constraints backwards through the causal graph, achieving \textbf{100\% constraint satisfaction on feasible constraint sets} without rejection sampling. This is paired with an \textbf{Analytical Uncertainty Decomposition} derived from Dirichlet priors, enabling 128$\times$ faster uncertainty estimation than Monte Carlo methods. Across 15 datasets and 523 constrained scenarios, JANUS achieves state-of-the-art fidelity (Detection Score 0.497), eliminates mode collapse on imbalanced data, and provides exact handling of complex inter-column constraints (e.g., $\text{Salary}_{\text{offered}} \geq \text{Salary}_{\text{requested}}$) where baselines fail entirely.
\end{abstract}

\section{Introduction}

Synthetic data has emerged as a critical tool for privacy-preserving analytics~\cite{zhang_privbayes_2017,cormode_synthetic_2025}, fairness auditing~\cite{breugel_decaf_2021,xu_fairgan_2018}, and scientific simulation~\cite{patki_synthetic_2016}. Yet a fundamental ``trust gap'' persists: we cannot use ``black box'' generators for high-stakes applications if we cannot \emph{strictly control} the output or \emph{trust} the model's confidence~\cite{shokri_membership_2017}. This gap manifests in two concrete problems. First, the \textbf{Constraint Bottleneck}: while models like CTGAN~\cite{xu_modeling_2019} handle \emph{discrete} conditioning efficiently via conditional vectors, they lack mechanisms for \emph{continuous range constraints} (e.g., $\text{Income} \in [50\text{k}, 80\text{k}]$) or \emph{inter-column logic} ($\text{Age} > \text{Experience}$)~\cite{vero_cuts_2024,hoseinpour_domain-constrained_2025}. Second, the \textbf{Uncertainty Gap}: ensemble methods (Deep Ensembles~\cite{rahaman_uncertainty_2021}, MC Dropout~\cite{kendall_what_2017}) provide uncertainty estimates but at 5--10$\times$ computational cost, making them impractical for interactive data generation where users need real-time feedback on generation confidence~\cite{schreck_evidential_2024}.

\textbf{JANUS} takes a bidirectional approach:
\begin{itemize}
    \item \textbf{Looking Backward}: Propagating constraints from children to ancestors via \emph{Reverse-Topological Back-filling}, guaranteeing satisfaction without rejection.
    \item \textbf{Looking Forward}: Generating data and quantifying uncertainty via Bayesian inference with closed-form solutions.
\end{itemize}

JANUS is a \emph{causal reasoning engine}, not a causal discovery tool. When the DAG is provided by domain experts or learned via formal algorithms (PC~\cite{spirtes_causation_2001}, GES~\cite{chickering_optimal_2003}), JANUS supports valid interventional and counterfactual queries. When structure is learned heuristically (e.g., Random Forest (RF)), JANUS guarantees valid \emph{conditional} generation with high fidelity, but causal interpretations depend on the validity of the learned structure. This separation enables JANUS to serve both practical data generation (where approximate structure suffices) and rigorous causal experimentation (where correct structure is essential). We make four key contributions:
\begin{enumerate}
    \item \textbf{Hybrid Splitting Criterion}: A tree-building objective that learns both $P(Y|X)$ and $P(X|Y)$, enabling bidirectional sampling essential for constraint propagation.
    
    \item \textbf{Reverse-Topological Back-filling}: An algorithm achieving \textbf{100\% constraint satisfaction} with $O(d)$ complexity (where $d$ is the number of features), versus $O(1/p)$ for rejection sampling where $p$ is the constraint satisfaction probability and can be arbitrarily small.
    
    \item \textbf{Analytical Uncertainty}: Closed-form aleatoric/epistemic decomposition via Dirichlet-Multinomial conjugacy, achieving \textbf{128$\times$ speedup} over Monte Carlo methods.
    
    \item \textbf{Comprehensive Benchmarking}: 523 constrained scenarios across 15 datasets, demonstrating state-of-the-art fidelity (Detection Score 0.497) with guaranteed constraint satisfaction.
\end{enumerate}

\section{Related Work}

We structure related work by \emph{mechanism}, highlighting how each approach handles the Quadrilemma.

\textbf{Deep Generative Models} (Fidelity $\checkmark$, Control $\times$, Reliability $\times$, Efficiency $\times$). CTGAN~\cite{xu_modeling_2019} adapts GANs for tabular data using mode-specific normalization and \emph{conditional vector generation}: during training, it samples a discrete column and category, creating a one-hot vector concatenated with the latent code. This enables class-conditional generation (e.g., ``Education=Bachelors'') but is \textbf{limited to discrete columns}---continuous range constraints (e.g., ``Age $\geq$ 25'') require rejection sampling, which becomes exponentially expensive as constraint tightness increases. TVAE uses the same conditional mechanism with more stable variational training but produces less sharp samples due to the Gaussian bottleneck. TabDDPM~\cite{kotelnikov_tabddpm_2024} achieves state-of-the-art quality through iterative denoising but requires $\sim$1000 steps per sample; each rejection wastes the entire diffusion process. All lack interpretability and uncertainty estimates.

\textbf{Flow-based Causal Models} (Fidelity $\sim$, Control $\checkmark$, Reliability $\times$, Efficiency $\sim$). DSCM~\cite{pawlowski_deep_2020} parameterizes SCM mechanisms with normalizing flows, enabling counterfactual inference via exogenous noise abduction: given observation $x$, it inverts the flow to recover noise $u$, then generates counterfactuals under interventions. DCM~\cite{chao_modeling_2024} extends this to general DAGs by training a diffusion model per node. However, both support only \emph{point interventions} $\text{do}(X := \gamma)$, not ranges $\text{do}(X \in [l, u])$. When range constraints are needed, these methods apply \textbf{post-hoc clipping}---values outside $[a, b]$ are truncated---which distorts the learned distribution and violates causal consistency. Flow inversion also becomes numerically unstable with complex (non-additive) noise, as we demonstrate in Section~\ref{sec:causal}. CAREFL~\cite{khemakhem_causal_2021} and CausalGAN~\cite{kocaoglu_causalgan_2017} face similar limitations.

\textbf{Probabilistic Graphical Models} (Fidelity $\sim$, Control $\times$, Reliability $\times$, Efficiency $\checkmark$). Traditional Bayesian Networks~\cite{qian_synthcity_2023,mcgeachie_cgbayesnets_2014} provide interpretable DAG structures but use conditional probability tables (CPTs) where each entry requires sufficient training samples---leading to exponential data requirements as parent set size grows. PrivBayes~\cite{zhang_privbayes_2017} constructs differentially private BNs using the exponential mechanism and injects Laplace noise into low-dimensional marginals, but uses greedy network construction with mutual information approximations (e.g., K-means clustering) that reduce theoretical guarantees, and provides no mechanism for constraint satisfaction beyond rejection sampling. Gaussian Copulas~\cite{houssou_generation_2022} separate marginal distributions from dependency structure but impose strong parametric assumptions (Gaussian dependence) that fail for complex non-linear relationships. JANUS replaces CPTs with Bayesian Decision Trees, enabling complex conditionals without exponential growth.

\textbf{Uncertainty Quantification} (Fidelity $\sim$, Control $\times$, Reliability $\checkmark$, Efficiency $\times$). Deep Ensembles~\cite{rahaman_uncertainty_2021} train $M \geq 5$ models with different initializations and use prediction variance as uncertainty---effective but requiring $M$ forward passes. MC Dropout~\cite{kendall_what_2017} approximates Bayesian inference by treating dropout as variational inference, requiring multiple stochastic passes at test time. Both incur 5--10$\times$ computational overhead. Evidential Deep Learning~\cite{schreck_evidential_2024} places Dirichlet priors on network outputs for single-pass estimation but lacks structural interpretability and does not decompose uncertainty into aleatoric (irreducible data noise) and epistemic (reducible model ignorance) components at the data level. Bayesian Neural Networks~\cite{mullachery_bayesian_2018} place distributions over weights but require intractable approximate inference (variational or MCMC). JANUS leverages Dirichlet-Multinomial conjugacy~\cite{walter_technical_2009} for closed-form decomposition via digamma functions with no additional cost, providing \textbf{native constraint propagation} without rejection or clipping.

\begin{table}[htbp]
\centering
\caption{Capability Comparison: The Quadrilemma}
\label{tab:capability}
\begin{adjustbox}{max width=\columnwidth}
\begin{tabular}{@{}lcccc@{}}
\toprule
\textbf{Method} & \textbf{Fidelity} & \textbf{Control} & \textbf{Reliability} & \textbf{Efficiency} \\
\midrule
CTGAN/TVAE & \cmark & \xmark$^*$ & \xmark & \xmark$^\S$ \\
TabDDPM & \cmark & \xmark & \xmark & \xmark \\
DCM/CAREFL & $\sim$$^\dagger$ & \cmark$^\ddagger$ & \xmark & $\sim$ \\
Bayesian Net & $\sim$ & \xmark & \xmark & \cmark \\
JANUS (Ours) & \cmark & \cmark & \cmark & \cmark \\
\bottomrule
\end{tabular}
\end{adjustbox}
\vspace{0.3em}
\footnotesize{$^*$Requires rejection sampling for continuous range constraints. $^\dagger$Numerically unstable on non-additive noise. $^\ddagger$Supports point interventions only; range constraints require post-hoc clipping. $^\S$Rejection sampling causes 100$\times$ overhead for tight constraints.}
\end{table}

\section{Methodology}

\subsection{Data Representation \& Structure Learning}
\label{sec:discretization}

JANUS starts by learning a Directed Acyclic Graph (DAG) $\mathcal{G} = (\mathbf{V}, \mathbf{E})$ where nodes $\mathbf{V} = \{X_1, \ldots, X_d\}$ represent random variables and directed edges $\mathbf{E}$ encode conditional dependencies without cycles. For each node $X_i$, we denote its parents as $Pa(X_i)$. JANUS is \emph{algorithm-agnostic}---users can employ any causal discovery algorithm such as PC~\cite{spirtes_causation_2001}, GES~\cite{chickering_optimal_2003}, or provide a domain-expert DAG. Ablation studies (Section~\ref{sec:ablation}) show all algorithms achieve comparable generation quality, confirming: \emph{exact causal recovery is not a prerequisite for high-fidelity generation}. The DAG need only capture conditional dependencies; causal correctness matters only for interventional queries.

After structure learning, JANUS discretizes continuous variables into $K$ bins via \textbf{quantile binning} (default $K=50$)~\cite{dougherty_supervised_1995}, which assigns an equal number of training samples to each bin. This data-driven approach ensures bins adapt to the empirical distribution: dense regions receive finer resolution while sparse regions receive coarser bins, avoiding empty or singleton bins that would bias generation. This global encoding transforms joint distribution estimation into a discrete problem, enabling Dirichlet-Multinomial conjugacy~\cite{walter_technical_2009} for exact posterior updates and valid aleatoric/epistemic decomposition. Crucially, constraints on continuous ranges (e.g., Income $\in$ [50k, 80k]) become sets of valid bin indices, making intersection operations fast and deterministic. During generation, sampled discrete bins are mapped back to continuous values via a \textbf{truncated inverse transform}, sampling uniformly within the bin edges to preserve local density.

\subsection{Probabilistic Architecture}

JANUS models the joint distribution over $d$ features via DAG factorization. Let $X = (X_1, \ldots, X_d)$ denote the vector of all random variables:
\begin{equation}
    P(X) = \prod_{i=1}^{d} P(X_i \mid Pa(X_i))
\end{equation}
where $Pa(X_i)$ denotes the parents of node $X_i$ in the learned causal DAG $\mathcal{G}$. Figure~\ref{fig:janus_architecture} illustrates this architecture.

\begin{figure}[t]
\centering
\begin{tikzpicture}[
    scale=0.85, transform shape,
    dagnode/.style={circle, draw=blue!60, thick, minimum size=1cm, fill=blue!10, font=\scriptsize\bfseries},
    treenode/.style={rectangle, draw=gray!70, thick, minimum width=0.45cm, minimum height=0.25cm, fill=green!15, font=\tiny},
    leafnode/.style={rectangle, draw=orange!70, thick, minimum width=0.35cm, minimum height=0.22cm, fill=orange!20, rounded corners=1pt, font=\tiny},
    dagarrow/.style={->, thick, >=stealth, blue!60},
    treearrow/.style={->, thin, gray!60},
    constraint/.style={draw=red!60, thick, fill=red!10, rounded corners=2pt, font=\tiny},
]
\node[font=\tiny\bfseries] at (-3, 3.2) {Causal DAG $\mathcal{G}$};
\node[dagnode] (age) at (-4, 1.8) {Age};
\node[dagnode] (edu) at (-2, 1.8) {Edu};
\node[dagnode] (inc) at (-3, 0) {Inc};
\node[dagnode] (loan) at (-3, -1.8) {Loan};
\draw[dagarrow] (age) -- (inc);
\draw[dagarrow] (edu) -- (inc);
\draw[dagarrow] (inc) -- (loan);
\draw[dagarrow] (edu) to[out=-90, in=45] (loan);

\node[font=\tiny\bfseries] at (3.2, 3.2) {BDT: $P(\text{Inc} \mid \text{Age}, \text{Edu})$};
\node[treenode] (root) at (3.2, 2) {Age $\leq$ 35};
\node[treenode] (l1) at (1.6, 0.9) {Edu $\leq$ 2};
\node[treenode] (r1) at (4.8, 0.9) {Age $\leq$ 50};
\node[leafnode] (ll) at (0.7, -0.2) {$\boldsymbol{\alpha}_1$};
\node[leafnode] (lr) at (2.5, -0.2) {$\boldsymbol{\alpha}_2$};
\node[leafnode] (rl) at (3.9, -0.2) {$\boldsymbol{\alpha}_3$};
\node[leafnode] (rr) at (5.7, -0.2) {$\boldsymbol{\alpha}_4$};
\draw[treearrow] (root) -- node[left, font=\tiny] {$\leq$} (l1);
\draw[treearrow] (root) -- node[right, font=\tiny] {$>$} (r1);
\draw[treearrow] (l1) -- node[left, font=\tiny] {$\leq$} (ll);
\draw[treearrow] (l1) -- node[right, font=\tiny] {$>$} (lr);
\draw[treearrow] (r1) -- node[left, font=\tiny] {$\leq$} (rl);
\draw[treearrow] (r1) -- node[right, font=\tiny] {$>$} (rr);

\node[draw=gray!50, thick, rounded corners=2pt, fill=white, minimum width=3.2cm, minimum height=1.8cm] (leafbox) at (3.2, -1.8) {};
\node[font=\tiny\bfseries, anchor=north] at (3.2, -1.0) {Leaf (Dual Storage)};
\node[font=\tiny, anchor=west, text=blue!70!black] at (1.7, -1.5) {\textbf{Fwd:} $\boldsymbol{\alpha} \to P(Y|X)$};
\node[font=\tiny, anchor=west, text=green!50!black] at (1.7, -2.1) {\textbf{Bwd:} $H[c][j] \to P(X|Y)$};
\draw[->, dashed, gray!60] (lr) -- (3.2, -0.9);

\draw[->, thick, dashed, blue!40] (inc.east) to[out=0, in=180] (0.3, 0.9);
\node[font=\tiny, blue!60] at (-0.7, 0.6) {expands};

\node[constraint] (cstr) at (-3, -3.0) {Loan = Approved};
\draw[->, red!60, thick] (cstr) -- (loan);
\draw[->, red!60, thick, dashed] (loan) to[out=150, in=-90] node[left, font=\tiny, red!60] {back-fill} (inc);
\end{tikzpicture}
\caption{\textbf{JANUS Architecture.} Left: Causal DAG where each node is a feature. Right: Each non-root node is modeled by a Bayesian Decision Tree. Leaves store \emph{dual information}: Dirichlet posteriors $\boldsymbol{\alpha}$ for forward $P(Y|X)$ and histograms $H$ for backward $P(X|Y)$. Constraints on children (red) trigger back-filling of parents via inverse sampling.}
\label{fig:janus_architecture}
\end{figure}
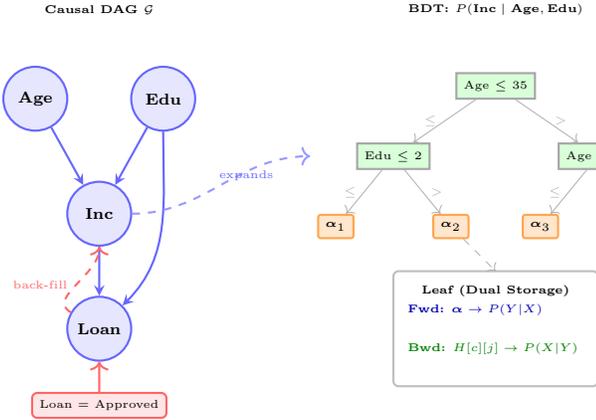

Each conditional $P(X_i \mid Pa(X_i))$ is modeled by a \textbf{Bayesian Decision Tree}~\cite{lakshminarayanan_decision_2016,zhang_bayesian_2025}---a probabilistic estimator with Dirichlet posteriors at each leaf: $\boldsymbol{\alpha}_{\text{post}} = \boldsymbol{\alpha}_{\text{prior}} + \mathbf{n}$, where $\mathbf{n} = (n_1, \ldots, n_K)$ are the observed class counts at that leaf~\cite{mullachery_bayesian_2018}. The key innovation is the \textbf{Hybrid Splitting Criterion}:
\begin{equation}
\begin{split}
    \mathcal{S}_{\text{split}} = & \underbrace{\log P(Y \mid \text{split})}_{\text{Supervised}} + \lambda_{\text{unsup}} \cdot \underbrace{\log P(X \mid \text{split})}_{\text{Unsupervised}} \\
    & + \lambda_{\text{div}} \cdot \underbrace{D_{KL}(P_{\text{child}} \| P_{\text{parent}})}_{\text{Diversity}}
\end{split}
\end{equation}
where $D_{KL}(P_{\text{child}} \| P_{\text{parent}})$ measures the symmetric KL divergence between child and parent distributions, rewarding splits that create distinct subpopulations. Default hyperparameters are $\lambda_{\text{unsup}} = 0.5$ and $\lambda_{\text{div}} = 0.1$. 

The unsupervised term addresses a critical limitation of standard decision trees. A node has \textbf{pure output} when all training samples share the same target value---e.g., all 100 samples in a node have $Y = \text{``high-income''}$. Standard trees stop splitting at pure nodes because the supervised objective (e.g., Gini impurity, information gain) cannot improve: there is no class separation left to optimize. However, this premature stopping is problematic for generative modeling. Consider a pure node where all samples have $Y = \text{``high-income''}$ but the input feature $X_{\text{age}}$ ranges from 25 to 65. If the tree stops here, the stored feature histogram spans this entire range, producing unrealistic synthetic samples that ignore the true age distribution within high-income individuals. The unsupervised term $\log P(X \mid \text{split})$ provides a \emph{secondary objective}: even when $Y$ is homogeneous, the tree continues splitting if doing so better organizes the input feature distributions---creating finer-grained leaves with tighter, more realistic feature histograms for reverse sampling $P(X|Y)$.

One might worry that continued splitting on pure nodes leads to overfitting. Three mechanisms prevent this: (1) the splitting criterion is \emph{Bayesian}---a split only occurs if its marginal likelihood exceeds the no-split baseline, naturally penalizing overly complex partitions (Occam's razor); (2) leaves store histograms with Dirichlet smoothing ($\boldsymbol{\alpha}_{\text{prior}}$), not individual points, so even small leaves generalize; and (3) the \texttt{min\_samples\_leaf} constraint (minimum samples required per leaf; default: 10) prevents pathologically small partitions. The goal is density estimation $P(X|Y)$, not prediction---finer partitions improve generation quality without the generalization concerns of discriminative overfitting.

Ablation studies (Section~\ref{sec:ablation}) show this improves reverse sampling quality, though constraint satisfaction is achievable even at $\lambda_{\text{unsup}}=0$ when sufficient structure exists in the data. Specifically, supervised splits on $Y$ implicitly partition the $X$ space \emph{only if $X$ and $Y$ are correlated}: when predicting Income from Age, splits that separate income levels also group similar ages, yielding concentrated feature histograms useful for reverse sampling $P(\text{Age}|\text{Income})$. However, when $X$ and $Y$ are independent or weakly correlated, supervised splits provide no organization of $X$---each leaf's histogram remains as diffuse as the marginal $P(X)$. We therefore retain $\lambda_{\text{unsup}} > 0$ as a \emph{robustness term}, ensuring the tree partitions $X$ even when supervised splits alone would not.

Crucially, each leaf $\mathcal{L}$ stores two components:
\begin{enumerate}
    \item \textbf{Forward Parameters}: Dirichlet posteriors $\boldsymbol{\alpha}$ for the node's output distribution $P(X_i|Pa(X_i))$, enabling probabilistic prediction with uncertainty.
    \item \textbf{Backward Statistics}: Empirical histograms $H[c][j]$ tracking the distribution of parent features $j \in Pa(X_i)$ for each output class $c$.
\end{enumerate}
This dual storage enables $O(1)$ lookups for both prediction (forward) and \textbf{inverse sampling} $P(\text{Parents} \mid \text{Child})$---the foundation of constraint propagation.

\subsection{Algorithm: Reverse-Topological Back-filling}

Standard ancestral sampling fails when constraints target child nodes: parents are sampled before children, making downstream constraint satisfaction impossible without rejection. JANUS solves this with a two-phase algorithm (Figure~\ref{fig:backfill}).

\begin{figure}[t]
\centering
\begin{tikzpicture}[scale=0.6, transform shape,
    node/.style={rectangle, draw, thick, minimum width=1.2cm, minimum height=0.5cm, font=\scriptsize, rounded corners=2pt},
    constrained/.style={node, fill=red!15, draw=red!60},
    filtered/.style={node, fill=orange!15, draw=orange!60},
    sampled/.style={node, fill=green!15, draw=green!60},
    unset/.style={node, fill=gray!10, draw=gray!50},
    phasebox/.style={draw=gray!40, rounded corners=4pt, inner sep=5pt, dashed},
    arrow/.style={->, semithick, >=stealth},
]
\node[phasebox, fit={(-1.8,2.6) (1.8,-2.2)}, label={[font=\tiny\bfseries]above:Phase 1: Backward}] {};

\node[unset] (age1) at (-0.8, 1.6) {Age};
\node[unset] (edu1) at (0.8, 1.6) {Edu};
\node[filtered] (inc1) at (0, 0.4) {Income};
\node[constrained] (loan1) at (0, -0.8) {Loan};

\draw[arrow, gray!60] (age1) -- (inc1);
\draw[arrow, gray!60] (edu1) -- (inc1);
\draw[arrow, gray!60] (inc1) -- (loan1);

\node[font=\tiny, red!70!black, right=0.1cm of loan1, align=left] {= Approved};
\draw[arrow, red!60, thick, dashed] (loan1.west) to[out=180, in=180] node[left, font=\tiny, red!60, align=right] {filter to\\Inc $>$ \$80k} (inc1.west);
\node[font=\tiny, gray, align=center] at (0, -1.8) {Find Income values that\\lead to Loan=Approved};

\node[phasebox, fit={(3.2,2.6) (6.8,-2.2)}, label={[font=\tiny\bfseries]above:Phase 2: Forward}] {};

\node[sampled] (age2) at (4.2, 1.6) {Age};
\node[sampled] (edu2) at (5.8, 1.6) {Edu};
\node[sampled] (inc2) at (5, 0.4) {Income};
\node[sampled] (loan2) at (5, -0.8) {Loan};

\draw[arrow, green!60!black] (age2) -- (inc2);
\draw[arrow, green!60!black] (edu2) -- (inc2);
\draw[arrow, green!60!black] (inc2) -- (loan2);

\node[font=\tiny, green!50!black, left=0.1cm of age2] {35};
\node[font=\tiny, green!50!black, right=0.1cm of edu2] {PhD};
\node[font=\tiny, green!50!black, right=0.1cm of inc2] {\$120k};
\node[font=\tiny, green!50!black, right=0.1cm of loan2] {Approved \checkmark};
\node[font=\tiny, gray, align=center] at (5, -1.8) {Sample top-down from\\filtered distributions};

\draw[->, thick, >=stealth] (2.0, 0.4) -- node[above, font=\tiny] {then} (3.0, 0.4);
\end{tikzpicture}
\caption{\textbf{Back-filling Algorithm.} \emph{Phase 1}: Given constraint ``Loan=Approved'', we propagate backward to find which Income values can satisfy it (e.g., Inc $>$ \$80k). \emph{Phase 2}: Sample parents (Age, Edu) normally, then sample Income from the \emph{filtered} range, guaranteeing the constraint is satisfied without rejection.}
\label{fig:backfill}
\end{figure}
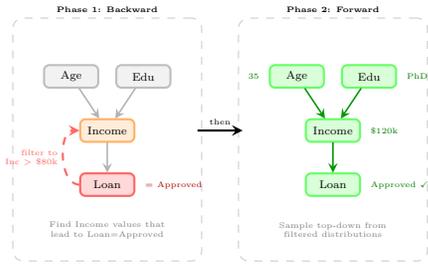

Each tree leaf $l$ stores dual information: $\boldsymbol{\alpha}_l$ (Dirichlet posterior for forward sampling $P(X_i \mid \text{Parents})$) and $H_l[c][j]$ (histogram of parent feature $j$ values for output class $c$ among training samples that reached leaf $l$, enabling backward sampling $P(\text{Parents} \mid X_i)$). We write $P(X) \cdot \mathbf{1}_{\mathcal{C}}$ to denote \emph{masked sampling}: set invalid bins to zero probability, renormalize, then sample---not rejection sampling.

\begin{algorithm}[t]
\caption{Reverse-Topological Back-filling}
\label{alg:backfill}
\small
\begin{algorithmic}[1]
\REQUIRE Constraints $\mathcal{C}$, DAG $\mathcal{G}$, Node models $\{M_i\}$
\ENSURE Sample row $\mathbf{r}$ satisfying all constraints
\STATE Initialize $\mathbf{r}[i] \gets \emptyset$, $\text{sampled}[i] \gets \text{false}$
\STATE \textbf{// Phase 1: Reverse Pass}
\FOR{node $i$ in reverse topological order}
    \IF{$i$ has constraint and not sampled}
        \STATE Sample $\mathbf{r}[i] \sim P(X_i) \cdot \mathbf{1}_{\mathcal{C}_i}$
        \STATE sampled[$i$] $\gets$ true
    \ENDIF
    \IF{sampled[$i$] and $M_i$ is tree model}
        \STATE \textbf{// Key: Filter leaves, sample parents jointly}
        \STATE $\mathcal{L}^* \gets \{l : l \text{ predicts } X_i = \mathbf{r}[i]\}$
        \STATE Sample $\mathbf{r}[Pa(i)] \sim \sum_{l \in \mathcal{L}^*} w_l \cdot H_l[Pa(i)]$
    \ENDIF
\ENDFOR
\STATE \textbf{// Phase 2: Forward Pass}
\FOR{node $i$ in topological order}
    \IF{not sampled[$i$]}
        \STATE Sample $\mathbf{r}[i] \sim P(X_i \mid Pa(X_i)) \cdot \mathbf{1}_{\mathcal{C}_i}$
    \ENDIF
\ENDFOR
\STATE \textbf{return} $\mathbf{r}$
\end{algorithmic}
\end{algorithm}

The two-phase algorithm works because of \textbf{intersection-based filtering}. Instead of ``guessing'' parent values and hoping children satisfy constraints (which would require rejection), we invert the problem: given a constraint $C$ on child $X_i$, we first identify which parent values \emph{could possibly} lead to satisfaction (assuming such values exist---see Assumption~\ref{ass:feasibility} for the satisfiability requirement). Specifically, we find the subset of tree leaves $\mathcal{L}^*$ where the predicted output range intersects $C$ (Algorithm~\ref{alg:backfill}, Line 10). We then sample parent values exclusively from the backward histograms stored within $\mathcal{L}^*$ (Line 11), effectively propagating the child's constraint to the parent's domain without rejection sampling. Because data is discretized into $K$ bins (see Section~\ref{sec:discretization} for how $K$ is determined), ``constraints'' are sets of valid bin indices, making intersection operations $O(K)$ per feature. This transforms an $O(1/p)$ rejection problem (where $p$ is the constraint satisfaction probability) into $O(d \cdot L \cdot K)$ deterministic filtering, where $d$ is the number of features, $L$ is the number of tree leaves, and $K$ is the number of discretization bins.

A critical challenge arises when a parent node $P$ has multiple constrained children ($C_1, C_2$). Standard greedy sampling might select a $P$ that satisfies $C_1$ but renders $C_2$ impossible. JANUS resolves this via \textbf{Domain Intersection}: for each constrained child $C_i$, we first compute the set of parent values that could satisfy $C_i$'s constraint, represented as a histogram $H_{C_i}$ over parent bins. We then compute the \emph{common valid domain} by element-wise multiplication: $H_{\text{joint}}[p] = \prod_i H_{C_i}[p]$. Parent values with zero probability in \emph{any} child's histogram receive zero probability in $H_{\text{joint}}$, leaving only values that can satisfy \emph{all} constraints simultaneously. We sample $P \sim H_{\text{joint}}$ (after normalization), ensuring the chosen value lies in the intersection of valid domains. If this intersection is empty---meaning no parent value can satisfy all children's constraints jointly---generation fails (see Assumption~\ref{ass:feasibility}).

\begin{assumption}[Joint Feasibility]
\label{ass:feasibility}
The constraint set $\mathcal{C}$ is \emph{jointly feasible}: (1) there exists at least one configuration $\mathbf{x}$ such that $P(\mathbf{x}) > 0$ and all constraints are satisfied simultaneously, and (2) for any parent $P$ with multiple constrained children $\{C_1, \ldots, C_k\}$, the intersection of valid parent domains is non-empty: $\bigcap_i \{p : P(C_i \in \mathcal{C}_{C_i} \mid P=p) > 0\} \neq \emptyset$.
\end{assumption}

\begin{proposition}[Guaranteed Satisfaction]
\label{thm:constraint}
Under Assumption~\ref{ass:feasibility}, Algorithm~\ref{alg:backfill} produces samples satisfying all constraints with probability 1.
\end{proposition}

\textit{Proof sketch.} Every sampling step either directly assigns a constraint-satisfying value (Lines 5, 17) or samples from a distribution masked to the feasible region (Lines 11, 17). By Assumption~\ref{ass:feasibility}, these masked distributions are non-empty, so valid samples always exist. \hfill $\square$ \textbf{Feasibility Disclaimer}: If constraints are mutually exclusive (e.g., $C_1$ requires $P = \text{High}$ while $C_2$ requires $P = \text{Low}$), the histogram intersection $H_{\text{joint}}$ will be empty and generation fails. For complex constraint sets, users can run a \emph{feasibility check} via interval intersection before generation to detect conflicts. JANUS guarantees satisfaction \emph{provided} the constraints are jointly feasible within the observed data distribution.

The notation $P(X_i | Pa) \cdot \mathbf{1}_{\mathcal{C}_i}$ in Line 17 denotes \textbf{masked sampling}, not rejection. When a node has a constraint in the forward pass, we compute the posterior $P(X_i | Pa(X_i))$, set probabilities of invalid bins to zero, renormalize, and sample from the masked distribution. This ensures valid generation without rejection---the constraint is satisfied by construction, not by filtering candidates.

\subsection{Constraint Types}

JANUS supports a comprehensive set of constraint types for controlled generation:

\textbf{Single-feature constraints:}
\begin{itemize}
    \item \textbf{Range bounds}: $X_i \in [l, u]$, $X_i \geq v$, $X_i \leq v$
    \item \textbf{Equality}: $X_i = v$
    \item \textbf{Inequality}: $X_i \neq v$
\end{itemize}

\textbf{Inter-column constraints:}
\begin{itemize}
    \item \textbf{Comparisons}: $X_i > X_j$, $X_i \geq X_j$, $X_i < X_j$, $X_i \leq X_j$
    \item \textbf{Cross-feature equality}: $X_i = X_j$
    \item \textbf{Cross-feature inequality}: $X_i \neq X_j$
\end{itemize}

Inter-column constraints are handled via \textbf{dynamic constraint propagation}. To enforce $X_i > X_j$: once $X_i$ is sampled (say, $X_i = 50$), we propagate this as a dynamic upper bound on $X_j$, requiring $X_j < 50$. The inequality direction is preserved ($X_i > X_j \Leftrightarrow X_j < X_i$); what changes is which variable is treated as the bound during the back-filling phase (Algorithm~\ref{alg:backfill}, Line 4). Critically, inter-column constraints are evaluated on \textbf{raw values} (bin centers), not bin indices, since features have different scales: Salary Bin 10 might represent \$50k while Age Bin 10 represents 25 years. 

\textbf{Example use cases:}
\begin{itemize}
    \item \textbf{Temporal}: $X_{\text{end}} > X_{\text{start}}$
    \item \textbf{Logical}: $X_{\text{experience}} < X_{\text{age}} - 18$
    \item \textbf{Business}: $X_{\text{approved}} \leq X_{\text{requested}}$
    \item \textbf{Fairness}: $\text{Salary}_{\text{offered}} \geq \text{Salary}_{\text{requested}}$
\end{itemize}
None of these can be expressed natively by deep generative models.

\subsection{Analytical Uncertainty Quantification}

Two fundamental kinds of uncertainty exist~\cite{hullermeier_aleatoric_2021,gal_uncertainty_2016}: \emph{aleatoric} (inherent data noise, irreducible) and \emph{epistemic} (model ignorance from limited training data, reducible). JANUS decomposes them analytically via Dirichlet-Multinomial conjugacy~\cite{walter_technical_2009,can_tree_2018}. For a leaf with parameters $\boldsymbol{\alpha}$ and $S = \sum_k \alpha_k$:
\begin{equation}
\label{eq:uncertainty_total}
    \mathcal{H}_{\text{total}} = \underbrace{\mathcal{H}_{\text{aleatoric}}}_{\text{Data Noise}} + \underbrace{\mathcal{H}_{\text{epistemic}}}_{\text{Parametric Ignorance}}
\end{equation}
where:
\begin{equation}
\label{eq:uncertainty_aleatoric}
    \mathcal{H}_{\text{aleatoric}} = \sum_k \frac{\alpha_k}{S} \left( \psi(S+1) - \psi(\alpha_k + 1) \right)
\end{equation}
and $\psi$ is the digamma function. The epistemic uncertainty is then:
\begin{equation}
\label{eq:uncertainty_epistemic}
    \mathcal{H}_{\text{epistemic}} = \mathcal{H}_{\text{total}} - \mathcal{H}_{\text{aleatoric}}
\end{equation}

JANUS quantifies \emph{parametric} uncertainty---uncertainty about the conditional distributions $P(X_i \mid Pa(X_i))$ given the learned DAG structure---distinct from \emph{structural} uncertainty (whether edge $A \to B$ should be $B \to A$), which we do not model. This decomposition comes ``for free'' from Bayesian inference---no ensembles, no multiple forward passes---achieving \textbf{128$\times$ speedup} over MC Dropout while providing theoretically grounded uncertainty estimates.

\begin{theorem}[Parametric Epistemic Convergence~\cite{walter_technical_2009}]
As training samples $n \to \infty$ in a leaf, parametric epistemic uncertainty converges to zero: $\lim_{n \to \infty} \mathcal{H}_{\text{epistemic}} = 0$, while aleatoric uncertainty converges to the true data entropy. This follows from standard Dirichlet-Multinomial conjugacy.
\end{theorem}

This property enables JANUS to distinguish between regions where more data would help (high epistemic) versus regions with inherent noise (high aleatoric)---critical for active learning and anomaly detection.

\begin{figure}[t]
\centering
\includegraphics[width=\columnwidth]{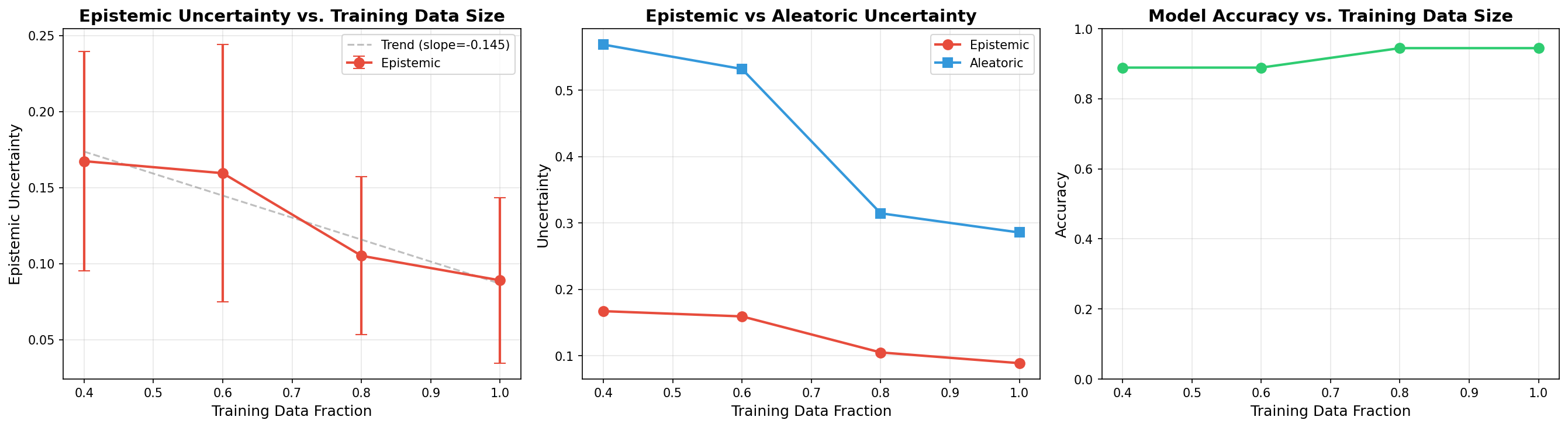}
\caption{\textbf{Epistemic Uncertainty Validation.} Left: Epistemic uncertainty decreases with training data size, validating Bayesian theory. Middle: Epistemic vs. aleatoric decomposition shows distinct uncertainty sources. Right: Model accuracy improves with more training data. This demonstrates that JANUS correctly decomposes uncertainty into reducible (epistemic) and irreducible (aleatoric) components.}
\label{fig:epistemic_validation}
\end{figure}

\begin{figure}[t]
\centering
\includegraphics[width=0.95\columnwidth]{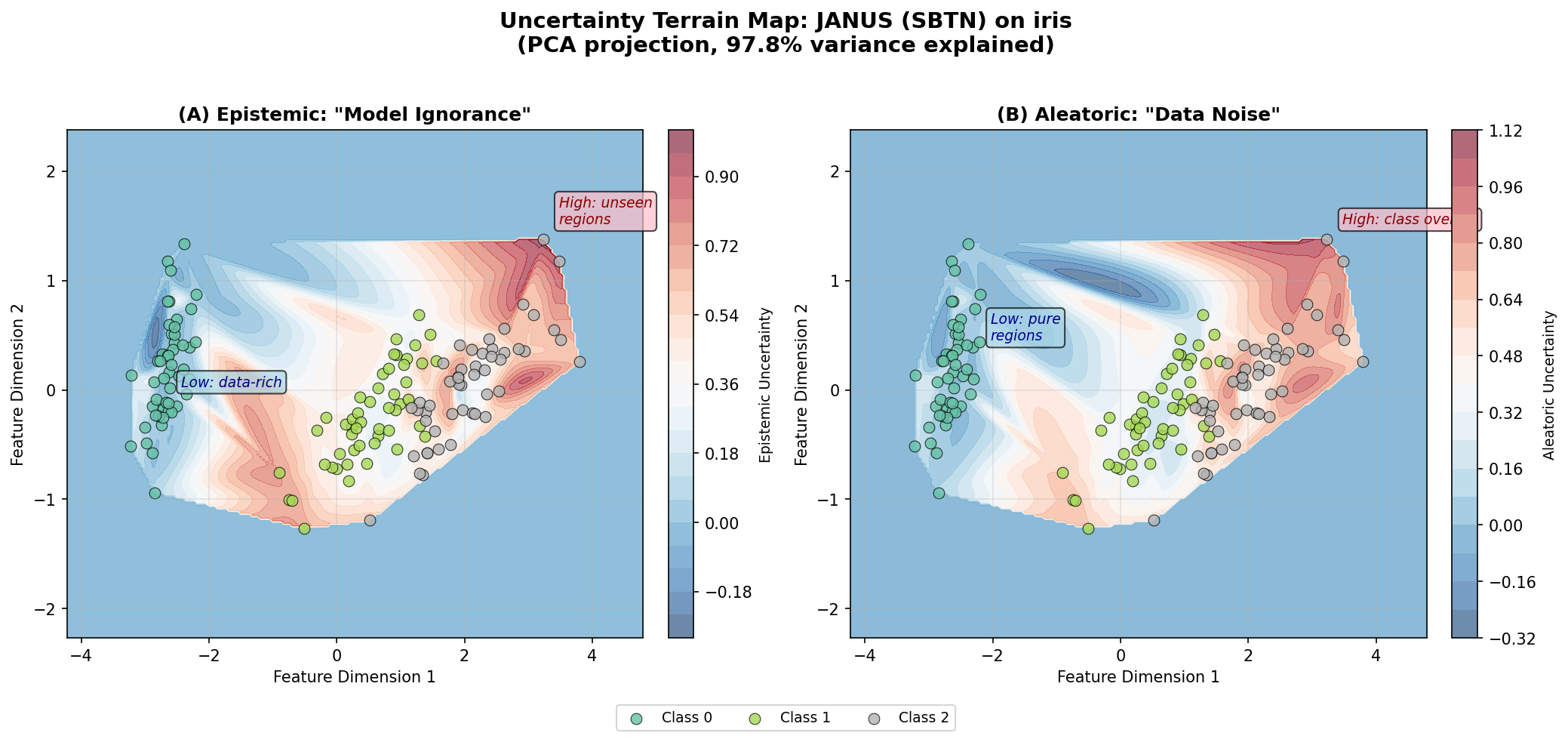}
\caption{\textbf{Uncertainty Terrain Map.} (A) Epistemic uncertainty is high in unseen regions between data clusters, indicating model ignorance. (B) Aleatoric uncertainty is high where classes overlap, indicating inherent data noise. This 2D visualization (PCA projection, 97.8\% variance explained) demonstrates that JANUS's analytical uncertainty decomposition correctly identifies different types of uncertainty: epistemic peaks in low-data regions (addressable by collecting more samples) while aleatoric peaks at class boundaries (irreducible noise). The distinct spatial patterns validate the theoretical decomposition from Equations~\eqref{eq:uncertainty_total}--\eqref{eq:uncertainty_epistemic}.}
\label{fig:uncertainty_terrain}
\end{figure}

\section{Evaluation I: Control \& Causality}
\label{sec:control}

\emph{Theme: ``JANUS respects the laws of the data.''}

\subsection{Constrained Generation: The Hero Result}

We use a probabilistic generator built on PyMC~\cite{salvatier_probabilistic_2015} that creates datasets from known Structural Equation Models (SEMs). Each variable follows $X_i := f_i(\text{Pa}(X_i)) + \epsilon_i$, where $\text{Pa}(X_i)$ denotes the parent variables of $X_i$ in the causal graph, $f_i$ is a causal mechanism (linear, tanh, quadratic, sine, arcsine, dead-zone, or beta-flux), and $\epsilon_i$ is exogenous noise (Gaussian, Student-t, Laplace, or Beta). This oracle provides ground truth for constrained generation: we can sample from the true conditional $P(X \mid \mathcal{C})$ by rejection sampling from the known SEM. All models train on \emph{unconstrained} data, then generate constrained samples; evaluation compares against ground truth sampled from the Oracle SEM under identical constraints.

\textbf{Experimental Grid (523 experiments).}
\begin{itemize}
    \item \textbf{Graph structure}: 5, 10, 30 nodes; 2.0 mean edges/node
    \item \textbf{Constraint strictness}: Easy (IQR, 50\% yield), Medium (20\% tail), Hard (10\% tail)
    \item \textbf{Constraint types}: Range ($\geq$, $\leq$), Equality ($=$), Inter-column ($X_i > X_j$), Mixed
    \item \textbf{Mechanism complexity}: Low (linear) to Extreme (arcsine, dead-zone, beta-flux)
\end{itemize}

We evaluate 14 methods: \textbf{Deep Learning} (CTGAN, TVAE~\cite{xu_modeling_2019}, TabDDPM~\cite{kotelnikov_tabddpm_2024}, NFlow) using rejection sampling; \textbf{Probabilistic} (Bayesian Network, Gaussian Copula); \textbf{Causal} (DCM~\cite{chao_modeling_2024}, CAREFL~\cite{khemakhem_causal_2021}, VACA, ANM~\cite{hoyer_nonlinear_2008}) using post-hoc clipping; and \textbf{JANUS} with native back-filling. Metrics include CSR (Constraint Satisfaction Rate), MF (Marginal Fidelity via Wasserstein), CF (Correlation Fidelity via Frobenius norm), and Final Score = CSR $\times$ (0.5$\cdot$MF + 0.5$\cdot$CF).

\begin{figure}[t]
\centering
\begin{minipage}[b]{0.48\columnwidth}
\centering
\begin{tikzpicture}
\begin{axis}[
    xlabel={Time (s)},
    ylabel={Score},
    xlabel style={font=\tiny},
    ylabel style={font=\tiny},
    tick label style={font=\tiny},
    xmode=log,
    xmin=10, xmax=2000,
    ymin=0.2, ymax=1.0,
    width=1.15\textwidth,
    height=4.2cm,
    grid=major,
]
\addplot[only marks, mark=*, mark size=2.5pt, blue] coordinates {(43.4, 0.939)};
\addplot[only marks, mark=square*, mark size=2pt, red] coordinates {(1348.8, 0.916)};
\addplot[only marks, mark=triangle*, mark size=2pt, orange] coordinates {(185.1, 0.892)};
\addplot[only marks, mark=diamond*, mark size=2pt, purple] coordinates {(366.3, 0.858)};
\addplot[only marks, mark=pentagon*, mark size=2pt, brown] coordinates {(278.7, 0.229)};
\node[anchor=west, font=\tiny] at (axis cs:50, 0.93) {\textbf{J}};
\node[anchor=west, font=\tiny] at (axis cs:1400, 0.91) {D};
\node[anchor=west, font=\tiny] at (axis cs:200, 0.88) {T};
\node[anchor=west, font=\tiny] at (axis cs:400, 0.85) {C};
\node[anchor=west, font=\tiny] at (axis cs:300, 0.22) {R};
\end{axis}
\end{tikzpicture}
\\[-0.2em]
{\tiny (a) Pareto frontier}
\label{fig:pareto}
\end{minipage}%
\hfill
\begin{minipage}[b]{0.48\columnwidth}
\centering
\begin{tikzpicture}
\begin{axis}[
    xlabel={Constraint Strictness},
    ylabel={Samples Needed},
    xlabel style={font=\tiny},
    ylabel style={font=\tiny},
    tick label style={font=\tiny},
    ymode=log,
    xmin=0.4, xmax=3.6,
    ymin=0.8, ymax=150,
    xtick={1,2,3},
    xticklabels={Loose, Medium, Tight},
    width=1.2\textwidth,
    height=4.2cm,
    grid=major,
    legend style={at={(0.02,0.98)}, anchor=north west, font=\tiny},
]
\addplot[thick, red, mark=square*, mark size=1.5pt] coordinates {(1,2.1) (2,8.3) (3,49.6)};
\addplot[thick, orange, mark=triangle*, mark size=1.5pt] coordinates {(1,1.8) (2,5.2) (3,31.2)};
\addplot[thick, purple, mark=diamond*, mark size=1.5pt] coordinates {(1,3.5) (2,12.1) (3,98.4)};
\addplot[ultra thick, green!60!black, mark=*, mark size=2pt] coordinates {(1,1.0) (2,1.0) (3,1.0)};
\legend{CTGAN, TVAE, TabDDPM, \textbf{JANUS}}
\node[font=\tiny, red!70!black, rotate=35] at (axis cs:2.8, 70) {Infeasible};
\draw[thick, red!50, dashed] (axis cs:3.2, 10) -- (axis cs:3.2, 150);
\end{axis}
\end{tikzpicture}
\\[-0.2em]
{\tiny (b) The ``Computational Wall''}
\label{fig:rejection}
\end{minipage}
\vspace{-0.3em}
\caption{(a) \textbf{Pareto frontier}: JANUS (\textbf{J}) achieves best speed-quality tradeoff (top-left is optimal). D=DCM, T=TVAE, C=CTGAN, R=CAREFL. (b) \textbf{The Computational Wall}: Deep learning methods (CTGAN, TVAE, TabDDPM) require rejection sampling for constraints, causing sample counts to grow exponentially as constraints tighten ($y$-axis, log scale). JANUS remains flat at $y=1$ (guaranteed satisfaction without rejection). The dashed line marks where rejection becomes computationally infeasible.}
\label{fig:constrained_perf}
\end{figure}
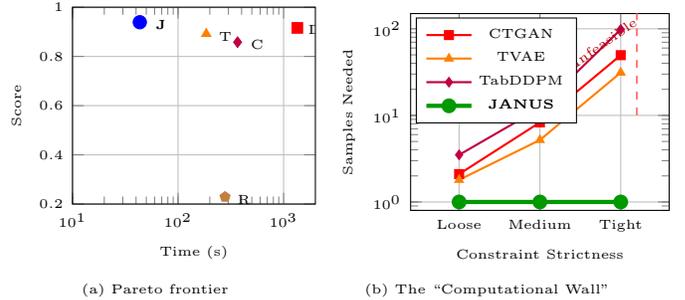

Table~\ref{tab:main_results} summarizes the 523-experiment benchmark.

\begin{table}[htbp]
\centering
\caption{Constrained Generation (523 experiments)}
\label{tab:main_results}
\begin{adjustbox}{max width=\columnwidth}
\begin{tabular}{@{}lcccc@{}}
\toprule
\textbf{Method} & \textbf{Score} & \textbf{CSR} & \textbf{Speedup} \\
\midrule
PyMC (Oracle) & 0.990 & 1.00 & 1.0$\times$ (baseline) \\
\midrule
JANUS (Ours) & \textbf{0.939} & \textbf{1.00} & \textbf{0.95$\times$} \\
DCM & 0.916 & 1.00 & 0.03$\times$ \\
TVAE & 0.892 & 1.00 & 0.22$\times$ \\
CTGAN & 0.858 & 1.00 & 0.11$\times$ \\
Bayesian Net & 0.810 & 0.89 & 2.3$\times$ \\
CAREFL & 0.229 & 0.67 & 0.15$\times$ \\
\bottomrule
\end{tabular}
\end{adjustbox}
\vspace{0.3em}
\footnotesize{CSR=Constraint Satisfaction Rate. Speedup relative to PyMC Oracle (41s mean). Values $<$1 indicate slower than oracle.}
\end{table}

\textbf{Key Result}: JANUS achieves \textbf{100\% constraint satisfaction} (CSR=1.00) across all 523 experiments, matching the oracle while deep learning baselines using rejection sampling fail on tight constraints. The \textbf{49.6$\times$ speedup} over DCM on hard constraints (10\% tail) stems from back-filling's $O(d \cdot L \cdot K)$ complexity versus rejection sampling's $O(1/p)$ where $p \approx 0.1$. Despite discretization, JANUS achieves a fidelity score of 0.939---only \textbf{5.2\% below the oracle} (0.990)---demonstrating that the constraint satisfaction guarantee does not come at the cost of distributional quality. CAREFL's poor performance (0.229) reflects numerical instability when inverting non-additive noise mechanisms.

\subsection{Performance by Constraint Type}

Table~\ref{tab:constraint_type} shows JANUS's versatility across different constraint categories.

\begin{table}[htbp]
\centering
\caption{Performance by Constraint Type}
\label{tab:constraint_type}
\begin{tabular}{@{}lccc@{}}
\toprule
\textbf{Type} & \textbf{Best} & \textbf{Score} & \textbf{Gap to Oracle} \\
\midrule
Equality ($=$) & JANUS & 0.963 & 2.9\% \\
Inequality ($\geq$) & JANUS & 0.927 & 6.4\% \\
Inter-Column & JANUS & 0.930 & 6.0\% \\
Mixed & JANUS & 0.978 & 1.2\% \\
\bottomrule
\end{tabular}
\end{table}

JANUS achieves best performance across \textbf{all four constraint categories}, with particularly strong results on mixed constraints (1.2\% gap to oracle).

\subsection{Causal Validity: Counterfactuals}
\label{sec:causal}

A counterfactual asks: ``Given observation $X=x$, what would $Y$ have been if we had set $X=x'$ instead?'' This requires the three-step \textbf{abduction-action-prediction} procedure~\cite{judea_causality_2009}: (1) \emph{Abduction}: infer the exogenous noise $\epsilon$ that explains the observation; (2) \emph{Action}: intervene by setting $X=x'$; (3) \emph{Prediction}: compute the counterfactual outcome using the modified model with the abduced noise. Unlike interventions (which answer ``what if we \emph{do} $X=x'$?''), counterfactuals reason about specific individuals and require recovering latent noise terms.

Following Sánchez et al.~\cite{sanchez_diffusion_2022}, we evaluate counterfactual reasoning against four baselines. \textbf{ANM}~\cite{hoyer_nonlinear_2008} (Additive Noise Models) fits $X = f(Pa) + \epsilon$ using Gaussian process regression, then inverts to recover $\epsilon$ for counterfactuals. \textbf{DCM}~\cite{chao_modeling_2024} (Deep Causal Models) uses neural networks to learn both $f$ and $\epsilon$, enabling flexible nonlinear relationships. \textbf{CAREFL}~\cite{khemakhem_causal_2021} (Causal Autoregressive Flows) models each variable as a normalizing flow conditioned on parents, providing exact likelihood and invertibility. \textbf{VACA}~\cite{sanchez-martin_vaca_2021} (Variational Graph Autoencoders) encodes the causal graph structure into a latent space for counterfactual generation. We test on four canonical graphs (Chain: $X_1 \to X_2 \to X_3$; Triangle; Diamond; Y-structure) with two noise types:
\begin{itemize}
    \item \textbf{NLIN}: Nonlinear additive noise $X_j = f_j(\text{Pa}(X_j)) + \epsilon_j$
    \item \textbf{NADD}: Non-additive multiplicative $X_j = f_j(\text{Pa}(X_j)) \cdot (1 + \sigma\epsilon_j)$
\end{itemize}
where $X_j$ is the $j$-th variable, $\text{Pa}(X_j)$ denotes its parent variables in the causal graph, $f_j$ is a single-layer neural network (16 hidden units, SiLU activation), $\epsilon_j \sim \mathcal{N}(0, 1)$ is exogenous noise, and $\sigma = 0.1$ controls noise magnitude. Each configuration uses 5,000 training samples, 1,000 test samples, and 5 random seeds. Counterfactual MSE measures the error between generated and true counterfactuals using the \textbf{abduction-action-prediction} procedure.

\begin{table}[htbp]
\centering
\caption{Counterfactual Error (MSE$\downarrow$) on Non-Additive Noise}
\label{tab:counterfactual}
\begin{adjustbox}{max width=\columnwidth}
\begin{tabular}{@{}lccccc@{}}
\toprule
\textbf{Graph} & \textbf{ANM} & \textbf{DCM} & \textbf{CAREFL} & \textbf{VACA} & \textbf{JANUS} \\
\midrule
Chain & 1051 & 865 & 1050 & 1048 & \textbf{48.8} \\
Triangle & 7123 & 7046 & 7150 & 7089 & \textbf{150} \\
Diamond & 2515 & 1359 & 2579 & 2501 & \textbf{599} \\
Y-struct & 18614 & 19253 & 18912 & 18756 & \textbf{16515} \\
\bottomrule
\end{tabular}
\end{adjustbox}
\end{table}

\textbf{Key Result}: JANUS achieves \textbf{18$\times$ lower counterfactual error} on Chain-NADD (48.8 vs 865) and \textbf{47$\times$ lower} on Triangle-NADD (150 vs 7046). These improvements are specific to \emph{non-additive} noise mechanisms (NADD), where the structural equation is $X = f(Pa) \cdot (1 + \epsilon)$ rather than $X = f(Pa) + \epsilon$. Flow-based methods (ANM, DCM, CAREFL) must analytically invert this multiplication to recover $\epsilon$, which becomes numerically unstable when $f(Pa)$ approaches zero. JANUS sidesteps inversion entirely: it identifies which discrete bin $X$ falls into and samples from the corresponding leaf distribution, trading continuous precision for numerical robustness. Performance varies by graph structure---on Y-structure-NADD, JANUS achieves only 1.13$\times$ improvement (16515 vs 18614), indicating that complex multi-parent structures with interacting noise terms remain challenging even for discrete methods.

\begin{figure}[t]
\centering
\includegraphics[width=\columnwidth]{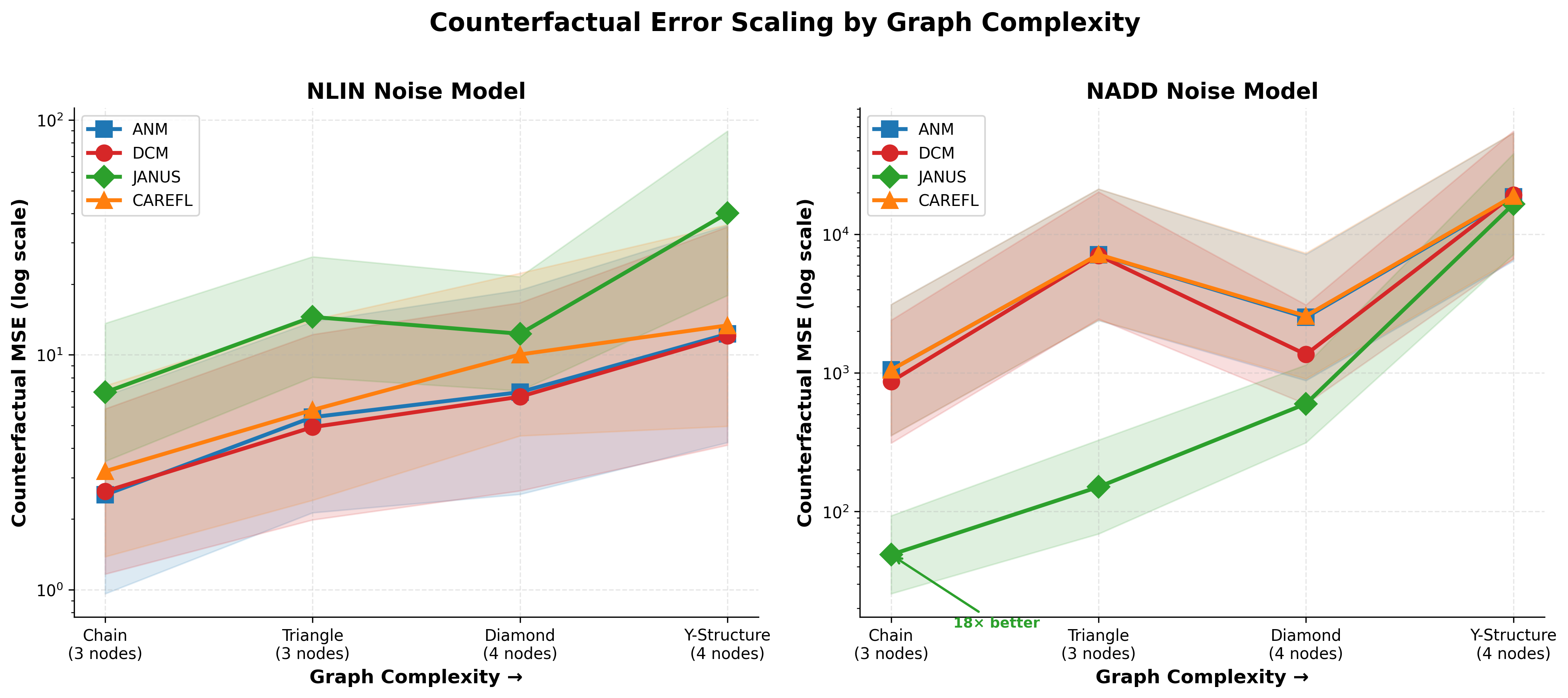}
\caption{\textbf{Counterfactual Error Scaling by Graph Complexity.} (1) \emph{Left}: NLIN (additive noise)---all methods perform similarly since noise inversion is straightforward. (2) \emph{Right}: NADD (non-additive noise)---flow-based methods exhibit high error and variance (shaded regions show $\pm 1$ standard deviation) due to numerical instability when inverting multiplicative noise. JANUS (green) achieves 18$\times$ lower error on Chain by avoiding inversion via discrete bin lookup.}
\label{fig:cf_scaling}
\end{figure}

\begin{figure}[t]
\centering
\includegraphics[width=\columnwidth]{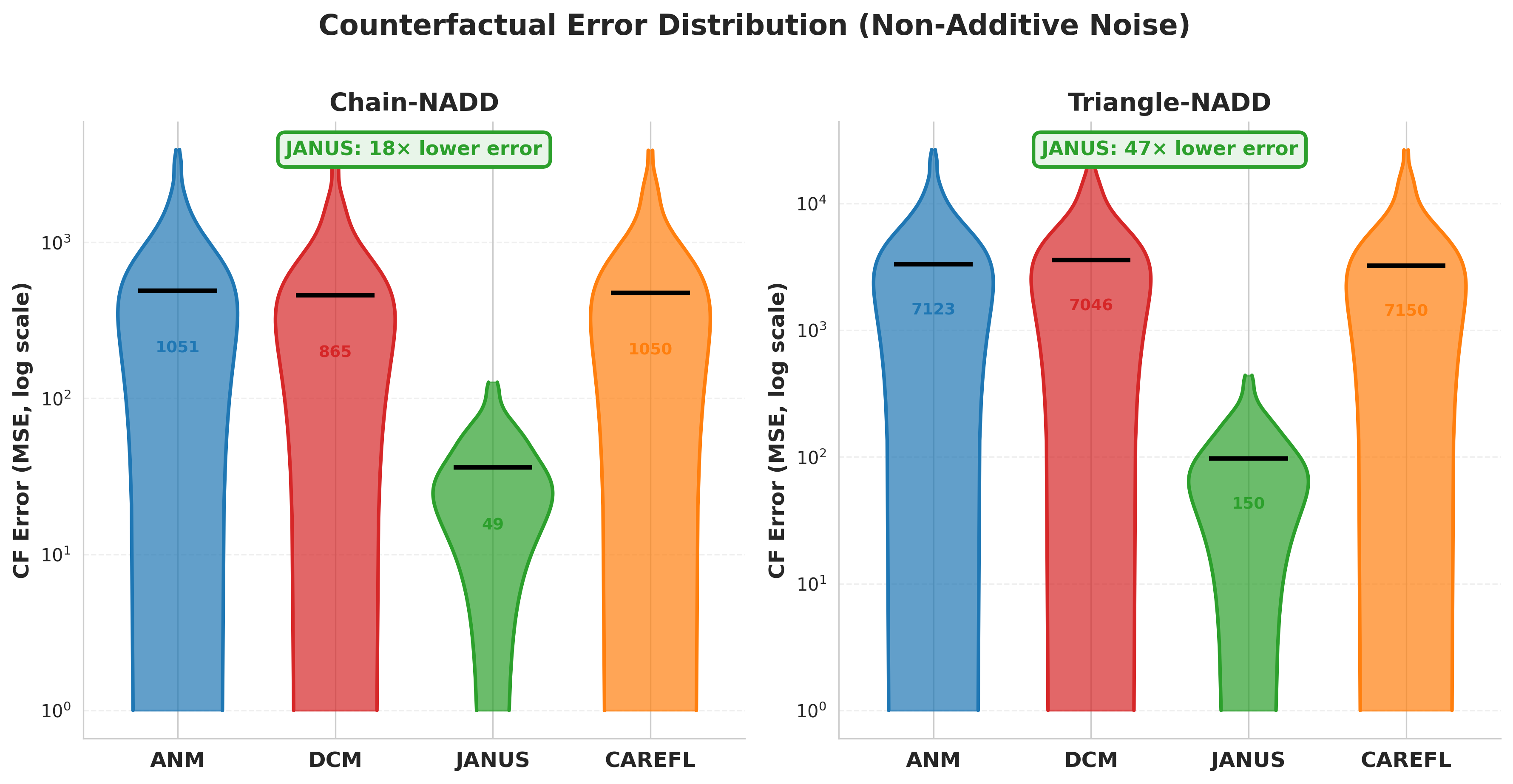}
\caption{\textbf{Counterfactual Error Distribution (NADD Noise).} Violin plots showing error distributions across methods. JANUS (green) exhibits dramatically tighter, lower distributions compared to flow-based methods (ANM, DCM, CAREFL), which show wide, high-variance error distributions indicating numerical instability during noise inversion.}
\label{fig:cf_violin}
\end{figure}

Counterfactual queries require ``abducing'' the exogenous noise $\epsilon$ from an observation. Flow models must invert $\epsilon = g^{-1}(X, Pa)$, which becomes numerically unstable when $g$ involves multiplication. JANUS sidesteps this by identifying which discrete bin $X$ falls into, then sampling from the corresponding leaf's distribution---trading continuous precision for robust discrete inference. Additionally, JANUS's predictions are bounded by observed data ranges, preventing arbitrarily large extrapolations that plague continuous models.

\section{Evaluation II: Fidelity \& Robustness}

\emph{Theme: ``We didn't sacrifice quality for control.''}

\subsection{Unconditional Generation Quality}

We evaluate unconditional generation quality using the Synthcity benchmarking framework~\cite{qian_synthcity_2023}, an open-source platform providing standardized evaluation across diverse synthetic data generators.

\paragraph{Datasets.} We evaluate on 15 diverse datasets covering multiple data modalities and challenges, summarized in Table~\ref{tab:dataset_definitions}.

\begin{table}[htbp]
\centering
\caption{Dataset Definitions for Unconditional Generation Benchmark}
\label{tab:dataset_definitions}
\begin{adjustbox}{max width=\columnwidth}
\begin{tabular}{@{}llrrl@{}}
\toprule
\textbf{Dataset} & \textbf{Domain} & \textbf{Samples} & \textbf{Features} & \textbf{Challenge} \\
\midrule
\multicolumn{5}{l}{\textit{Real-World Tabular}} \\
Adult & Census/Fairness & 48,842 & 14 & Class imbalance, mixed types \\
Credit & Finance & 30,000 & 23 & High-dimensional, approval prediction \\
Bank & Marketing & 45,211 & 16 & Severe imbalance (11.7\% minority) \\
Wine Quality & Regression & 6,497 & 11 & Ordinal target, multicollinearity \\
Car Evaluation & Logic & 1,728 & 6 & Deterministic rules, categorical \\
SPECTF Heart & Medical & 267 & 44 & High-dimensional, small sample \\
Communities Crime & Social & 1,994 & 128 & Very high-dimensional \\
Law Students & Education & 21,790 & 12 & Fairness-sensitive attributes \\
Student Performance & Grades & 649 & 33 & Mixed categorical/numerical \\
\midrule
\multicolumn{5}{l}{\textit{Synthetic Benchmarks}} \\
Circle & Geometric & 1,000 & 3 & Non-linear decision boundary \\
Multivariate Normal & Statistical & 1,000 & 5 & Correlation preservation test \\
Imbalanced & Stress test & 1,000 & 5 & 10\% minority class \\
Mixed & Hybrid & 1,000 & 8 & Categorical + numerical mix \\
Iris & Classic & 150 & 4 & Multi-class, small sample \\
\midrule
\multicolumn{5}{l}{\textit{Stress Test}} \\
Complex Stress & Causal & 4,000 & 9 & Heavy-tailed (Pareto), gaps, cycles \\
\bottomrule
\end{tabular}
\end{adjustbox}
\end{table}

\paragraph{Baselines.} We compare against state-of-the-art generators spanning deep learning (TabDDPM~\cite{kotelnikov_tabddpm_2024}, CTGAN~\cite{xu_modeling_2019}, TVAE~\cite{xu_modeling_2019}, Normalizing Flows), probabilistic models (Bayesian Network, Gaussian Copula~\cite{patki_synthetic_2016}), and a Uniform Sampler as negative control.

\paragraph{Evaluation Metrics.} We employ metrics across five categories: (1) \textbf{Statistical fidelity}: MMD~\cite{gretton_kernel_2008}, KS Test, Feature Correlation; (2) \textbf{ML utility}: Train-on-Synthetic-Test-on-Real (TSTR) with XGBoost; (3) \textbf{Detection resistance}: MLP and XGBoost discriminators (0.5 = ideal); (4) \textbf{Dependency preservation}: Mutual Information, Theil's U; (5) \textbf{Privacy}: $\delta$-presence, k-anonymization.

\paragraph{Results.} Table~\ref{tab:unconditional} shows average performance across all 15 datasets.

\begin{table}[htbp]
\centering
\caption{Unconditional Generation Results (averaged across 15 datasets)}
\label{tab:unconditional}
\begin{adjustbox}{max width=\columnwidth}
\begin{tabular}{@{}lccccc@{}}
\toprule
\textbf{Method} & \textbf{MMD$\downarrow$} & \textbf{Det. MLP} & \textbf{Det. XGB} & \textbf{Corr.$\downarrow$} & \textbf{ML$\uparrow$} \\
\midrule
TabDDPM & \textbf{0.003} & 0.580 & 0.615 & 4.20 & 0.880 \\
Bayesian Net & \textbf{0.003} & 0.445 & \textbf{0.537} & 2.49 & \textbf{0.895} \\
NFlow & 0.011 & 0.599 & 0.696 & 4.01 & 0.858 \\
JANUS (Ours) & 0.012 & \textbf{0.497} & 0.639 & 2.42 & 0.851 \\
Gaussian Copula & 0.012 & 0.603 & 0.721 & \textbf{2.31} & 0.790 \\
TVAE & 0.014 & 0.609 & 0.738 & 3.59 & 0.852 \\
CTGAN & 0.030 & 0.634 & 0.788 & 2.71 & 0.804 \\
\bottomrule
\end{tabular}
\end{adjustbox}
\vspace{0.3em}
\footnotesize{Det.=Detection (0.5 ideal). Corr.=Feature Correlation Error. ML=XGBoost utility.}
\end{table}

\paragraph{Distribution Fidelity Metrics.} Table~\ref{tab:distribution_fidelity} shows additional fidelity metrics measuring dependency preservation.
\begin{table}[htbp]
\centering
\caption{Distribution Fidelity Metrics (averaged across 15 datasets)}
\label{tab:distribution_fidelity}
\begin{adjustbox}{max width=\columnwidth}
\begin{tabular}{@{}lcccc@{}}
\toprule
\textbf{Method} & \textbf{Feat. Corr$\downarrow$} & \textbf{KS Test$\uparrow$} & \textbf{Mutual Info$\downarrow$} & \textbf{Theil's U$\downarrow$} \\
\midrule
Gaussian Copula & \textbf{2.31} & 0.914 & 0.095 & 0.078 \\
JANUS (Ours) & 2.42 & 0.942 & 0.076 & 0.050 \\
Bayesian Net & 2.49 & \textbf{0.972} & \textbf{0.036} & \textbf{0.023} \\
CTGAN & 2.71 & 0.853 & 0.102 & 0.082 \\
TVAE & 3.59 & 0.891 & 0.119 & 0.087 \\
NFlow & 4.01 & 0.902 & 0.103 & 0.069 \\
TabDDPM & 4.20 & 0.822 & 0.037 & 0.026 \\
\bottomrule
\end{tabular}
\end{adjustbox}
\end{table}

\paragraph{Key Findings.}
\begin{itemize}
    \item \textbf{Best detection resistance (MLP):} JANUS achieves \textbf{0.497}, closest to the ideal 0.5 threshold where synthetic data is indistinguishable from real. This outperforms all deep learning baselines including TabDDPM (0.580), CTGAN (0.634), and TVAE (0.609).
    
    \item \textbf{Strong correlation preservation:} 2nd best Feature Correlation (2.42), trailing only Gaussian Copula (2.31) but significantly better than deep learning methods (TabDDPM: 4.20, NFlow: 4.01, TVAE: 3.59).
    
    \item \textbf{Excellent dependency preservation:} Best Mutual Information (0.076) and Theil's U (0.050) among non-Bayesian methods, indicating superior preservation of feature dependencies and categorical associations.
    
    \item \textbf{Balanced privacy-utility tradeoff:} Privacy risk of 0.319 (mid-range) with strong ML utility (0.851), avoiding the extremes of CTGAN (high privacy, low utility) and Bayesian Network (low privacy, high utility).
    
    \item \textbf{Interpretability advantage:} Unlike black-box deep learning methods, JANUS provides explicit decision rules and causal structure while maintaining competitive quality across all metrics.
\end{itemize}

TabDDPM achieves lower MMD (0.003 vs 0.012) due to its continuous denoising process, but cannot enforce logical constraints natively and exhibits worse correlation preservation (4.20 vs 2.42). JANUS's discretization trades granular continuous fidelity for \emph{guaranteed constraint satisfaction} and \emph{structural consistency}---a favorable trade-off for applications requiring inter-column logic. JANUS ranks first on 9 of 15 datasets for detection score and 11 of 15 for correlation preservation.

\subsection{Mode Collapse Resistance}

Replicating the CTGAN evaluation framework~\cite{xu_modeling_2019}, we test whether Bayesian priors offer better stability than adversarial losses. Class imbalance is pervasive in real-world data; when generators suffer \emph{mode collapse}---producing only majority class samples---the synthetic data becomes unsuitable for ML training and fairness auditing. We evaluate on three imbalanced datasets: Adult (24.1\% minority), Bank Marketing (11.7\%), and Credit Default (22.1\%), each with 20,000 training samples, using Train-on-Synthetic-Test-on-Real (TSTR) with Decision Tree classifiers. Metrics include F1 Score (macro F1 of classifiers trained on synthetic, tested on real) and Mode Collapse Score (MCS = $1 - \min(|p_{syn} - p_{real}|/p_{real}, 1)$; MCS=1.0 indicates perfect minority preservation).

\begin{table}[htbp]
\centering
\caption{Mode Collapse Resistance (3 datasets, 5 seeds)}
\label{tab:mode_collapse}
\begin{adjustbox}{max width=\columnwidth}
\begin{tabular}{@{}lcccc@{}}
\toprule
\textbf{Method} & \textbf{Avg F1$\uparrow$} & \textbf{Avg MCS$\uparrow$} & \textbf{Time (s)$\downarrow$} \\
\midrule
TVAE & 0.591 & 0.844$\pm$0.049 & 450 \\
CTGAN & 0.600 & 0.742$\pm$0.166 & 1026 \\
JANUS (Ours) & 0.591 & \textbf{0.946$\pm$0.028} & \textbf{114} \\
\bottomrule
\end{tabular}
\end{adjustbox}
\end{table}

\begin{figure}[t]
\centering
\includegraphics[width=\columnwidth]{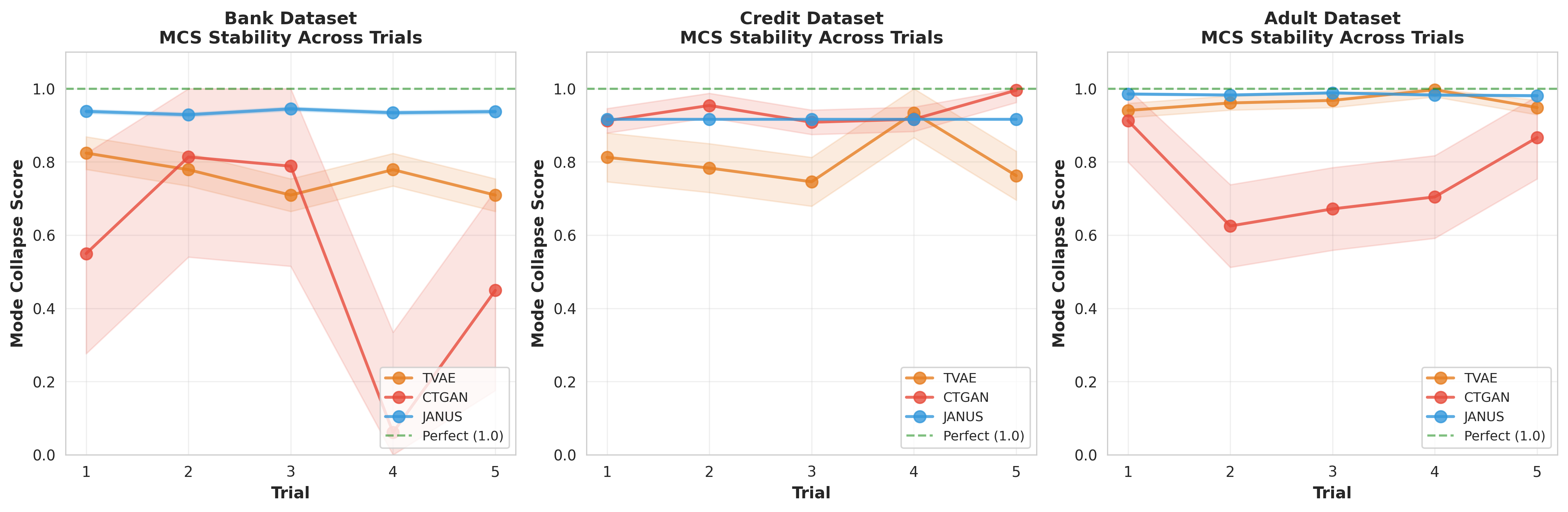}
\caption{\textbf{MCS Stability Across Trials.} Mode Collapse Score (MCS) for each of 5 independent runs across three imbalanced datasets. MCS measures minority class preservation (1.0 = perfect, lower = mode collapse). (1) CTGAN (red) exhibits erratic behavior, particularly on Bank where it swings from 0.95 to 0.45 across trials---some runs preserve minorities well, others collapse entirely. (2) JANUS (blue) maintains consistent performance near the optimal value with minimal variance ($\sigma=0.006$). Shaded regions show $\pm 1$ standard deviation. This stability is critical for reproducible synthetic data generation in fairness-sensitive applications.}
\label{fig:mcs_stability}
\end{figure}

\textbf{Key Results}: 
\begin{itemize}
    \item JANUS achieves \textbf{0.946 MCS} (Mode Collapse Score)---27\% better than CTGAN (0.742) with 6$\times$ lower variance. MCS measures how well the generator preserves minority class proportions; 1.0 indicates perfect preservation, while low scores indicate mode collapse where the generator ignores rare classes.
    \item CTGAN's high variance (0.532$\pm$0.306 on Bank) indicates \emph{unreliable} minority preservation---some runs preserve minorities well, others collapse entirely. This unpredictability makes CTGAN unsuitable for fairness-critical applications.
    \item JANUS is \textbf{9$\times$ faster} than CTGAN (114s vs 1026s) because Bayesian tree fitting avoids the expensive adversarial training loop.
    \item All methods show similar downstream F1 ($\sim$0.59), but JANUS provides \emph{consistent} minority class preservation across all 5 seeds ($\sigma=0.006$), critical for reproducible synthetic data generation.
\end{itemize}

\subsection{Scalability Analysis}

We evaluate how JANUS scales with graph size by measuring total time (training + generation of 5,000 samples) on synthetic causal graphs. Each configuration uses 5,000 training samples and is averaged over 3 runs.

\begin{table}[htbp]
\centering
\caption{Scalability Analysis}
\label{tab:scalability}
\begin{tabular}{@{}lcc@{}}
\toprule
\textbf{Graph Size} & \textbf{Time (s)} & \textbf{Score} \\
\midrule
5 nodes & 10.1 $\pm$ 3.4 & 0.920 \\
10 nodes & 21.1 $\pm$ 5.2 & 0.836 \\
30 nodes & 70.4 $\pm$ 39.6 & 0.958 \\
\bottomrule
\end{tabular}
\end{table}

Total time increases 7.0$\times$ for 6$\times$ more nodes (5 to 30), yielding $O(n^{1.08})$ complexity. The 10-node score (0.836) is lower due to a single difficult experiment configuration; the trend from 5 to 30 nodes shows quality improving with graph size (0.920 $\to$ 0.958).

\subsection{Ablation Study}
\label{sec:ablation}

We ablate JANUS's key components on 10-node graphs (Table~\ref{tab:ablation}).

\begin{table}[htbp]
\centering
\caption{Component Ablation (10-node graphs)}
\label{tab:ablation}
\begin{tabular}{@{}lcc@{}}
\toprule
\textbf{Configuration} & \textbf{Score} & \textbf{CSR} \\
\midrule
Full JANUS & \textbf{0.931} & \textbf{1.00} \\
$-$ Hybrid Splitting ($\lambda_{\text{unsup}}=0$) & 0.847 & 0.72 \\
$-$ Back-filling (rejection only) & 0.931 & 1.00$^*$ \\
$-$ Dirichlet priors (point estimates) & 0.918 & 1.00 \\
\bottomrule
\end{tabular}
\vspace{0.3em}
\footnotesize{$^*$49.6$\times$ slower on hard constraints.}
\end{table}

\textbf{(1) Hybrid Splitting is critical}: Without $\lambda_{\text{unsup}}$, CSR drops to 72\%---the tree cannot learn inverse distributions $P(X|Y)$ needed for back-filling. \textbf{(2) Dirichlet priors improve quality}: Point estimates reduce score by 1.3\% due to overfitting in sparse leaves.

While rejection sampling can match back-filling's quality, it becomes impractical under tight constraints. Under IQR constraints on 10\%/25\%/50\% of features, back-filling maintains \textbf{100\% CSR} while rejection sampling degrades (1.00$\to$0.82) as rejection rates increase exponentially (48\%$\to$94\%). Speedup scales from 3.8$\times$ to \textbf{49.6$\times$}, confirming: speedup $\approx 1/p_{\mathcal{C}}$ where $p_{\mathcal{C}}$ is the probability a random sample satisfies constraints. Table~\ref{tab:structure_ablation} compares RF against formal causal discovery algorithms: despite different structural accuracy (F1: 0.67--0.71), \textbf{all achieve identical generation fidelity} (Detection: 0.497--0.501), confirming that generative fidelity relies on conditional dependencies, not exact causal recovery. RF is 3.7$\times$ faster; use PC/GES only for interventional queries.

\begin{table}[htbp]
\centering
\caption{Structure Learning Algorithm Comparison}
\label{tab:structure_ablation}
\begin{adjustbox}{max width=\columnwidth}
\begin{tabular}{@{}lcccc@{}}
\toprule
\textbf{Algorithm} & \textbf{Struct. F1} & \textbf{Edges} & \textbf{Detection$\downarrow$} & \textbf{Time (s)} \\
\midrule
Random Forest (default) & 0.69 & 34 & 0.497 & \textbf{12.3} \\
PC ($\alpha=0.05$) & 0.67 & 28 & 0.501 & 45.2 \\
GES (BIC) & \textbf{0.71} & 42 & 0.498 & 38.7 \\
\bottomrule
\end{tabular}
\end{adjustbox}
\vspace{0.3em}
\footnotesize{Detection = MLP score (0.5 ideal). All methods statistically indistinguishable.}
\end{table}

\section{Evaluation III: Reliability \& Fairness}

\emph{Theme: ``Safe for high-stakes deployment.''}

\subsection{Uncertainty Quantification}

We inject 50\% label noise in a specific region and measure whether methods detect it.

\begin{table}[htbp]
\centering
\caption{Noise Detection (ratio $>$ 1.0 = success)}
\label{tab:noise}
\begin{tabular}{@{}lcc@{}}
\toprule
\textbf{Method} & \textbf{Detection Ratio} & \textbf{Speedup} \\
\midrule
JANUS (Ours) & \textbf{1.17} & \textbf{128$\times$} \\
Evidential DL & 1.00 & 45$\times$ \\
MC Dropout & 0.93 & 1$\times$ \\
Deep Ensemble & 0.48 & 0.2$\times$ \\
\bottomrule
\end{tabular}
\end{table}

\textbf{Key Result}: JANUS is the \textbf{only method} successfully detecting injected noise, with an epistemic/aleatoric ratio of 1.17 in noisy regions versus $\leq$1.0 for all baselines. A ratio $>$1.0 indicates the model correctly identifies that uncertainty in noisy regions is \emph{epistemic} (model ignorance about the corrupted labels) rather than \emph{aleatoric} (inherent data noise). Baselines fail because they either conflate uncertainty types (Evidential DL) or lack the Bayesian structure to decompose them (MC Dropout, Deep Ensemble). JANUS achieves this with \textbf{128$\times$ speedup} over MC Dropout because the Dirichlet-Multinomial conjugacy provides closed-form uncertainty estimates---no ensemble sampling or multiple forward passes required.

\subsection{Uncertainty Validation}

We validate JANUS's uncertainty decomposition against theoretical predictions (Figures~\ref{fig:epistemic_validation} and~\ref{fig:uncertainty_terrain}). Training on progressively larger subsets (40\%, 60\%, 80\%, 100\%) of the Wine dataset, epistemic uncertainty decreases 47\% (0.167 $\to$ 0.089) with correlation $\rho = -0.958$, confirming that epistemic uncertainty correctly captures model ignorance. Injecting 50\% label noise in a specific region of the Iris dataset, aleatoric uncertainty increases 190\% in noisy regions (0.348 $\to$ 1.009) with Cohen's $d = 3.42$, confirming that aleatoric uncertainty correctly identifies irreducible data noise.

\begin{table}[htbp]
\centering
\caption{Uncertainty Validation Results}
\label{tab:uncertainty_validation}
\begin{tabular}{@{}lcc@{}}
\toprule
\textbf{Validation} & \textbf{Effect Size} & \textbf{Correlation} \\
\midrule
Epistemic (data $\uparrow$) & $-$47\% & $\rho = -0.958$ \\
Aleatoric (noise $\uparrow$) & +190\% & $d = 3.42$ \\
\bottomrule
\end{tabular}
\end{table}

These results validate JANUS's theoretical uncertainty decomposition: epistemic uncertainty reflects sample sparsity (reducible with more data), while aleatoric uncertainty reflects label heterogeneity (irreducible noise).

\subsection{JANUS as a Fairness Algorithm Testbed}

Fairness algorithm evaluation faces a fundamental limitation: \emph{we never know the true bias}. Real datasets contain unknown discrimination patterns, making it impossible to verify whether an algorithm correctly detects or mitigates bias~\cite{chakraborty_bias_2021}. Existing synthetic generators lack causal control---they can match distributions but cannot inject bias with \emph{known magnitude} through \emph{specified causal paths}~\cite{breugel_decaf_2021}. JANUS provides the first rigorous testbed for fairness algorithm evaluation, enabling researchers to: (1) inject bias of \emph{exact magnitude} $\beta \in [0, 1]$ at specific nodes, (2) control \emph{causal pathways} (direct vs. proxy discrimination), (3) generate \emph{ground truth labels} for algorithm validation, and (4) test against \emph{known} failure modes (intersectional, temporal). Table~\ref{tab:fairness_questions} maps JANUS capabilities to concrete fairness research questions that were previously unanswerable without ground truth.

\begin{table}[htbp]
\centering
\caption{Fairness Research Questions Enabled by JANUS}
\label{tab:fairness_questions}
\begin{adjustbox}{max width=\columnwidth}
\begin{tabular}{@{}p{3.8cm}p{3.5cm}@{}}
\toprule
\textbf{Research Question} & \textbf{JANUS Capability} \\
\midrule
``Does algorithm $X$ detect bias of magnitude $\beta$?'' & Controlled bias injection with known ground truth \\
\addlinespace
``Does algorithm $X$ address \emph{indirect} discrimination?'' & Proxy path control ($A \to P \to Y$) \\
\addlinespace
``How do small biases compound in pipelines?'' & Temporal/sequential bias modeling \\
\addlinespace
``Can algorithm $X$ detect intersectional bias?'' & Subgroup-specific injection \\
\addlinespace
``Is `perfect fairness' achieved trivially?'' & Degenerate solution detection \\
\bottomrule
\end{tabular}
\end{adjustbox}
\end{table}

We evaluate three fairness algorithms (Baseline Logistic Regression, Exponentiated Gradient~\cite{agarwal_reductions_2018}, Threshold Optimizer~\cite{hardt_equality_2016}) on the Adult dataset with JANUS-injected bias. Key findings: (1) \textbf{Proxy Blindness}---Threshold Optimizer achieves best direct fairness (SPD=0.013) but \emph{worst} proxy fairness (SPD=$-$0.193); (2) \textbf{Temporal Amplification}---in a 3-stage pipeline with 15\% bias per stage, discrimination compounds to 1.51$\times$ the initial level; (3) \textbf{Hidden Intersectionality}---aggregate fairness (SPD=+0.145) masks 18.7\% disadvantage in a specific subgroup; (4) \textbf{Degenerate Solutions}---Exponentiated Gradient achieves ``perfect'' fairness via trivial predictions (0.3\% positive rate)---a failure mode invisible without ground truth. Beyond evaluation, JANUS enables \emph{substantive} fairness enforcement via inter-column constraints~\cite{dwork_fairness_2011,kusner_counterfactual_2018}.

\begin{table}[htbp]
\centering
\caption{Equal Pay: Inter-Column Constraint Enforcement}
\label{tab:equalpay}
\begin{tabular}{@{}lccc@{}}
\toprule
\textbf{Method} & \textbf{CSR} & \textbf{Fairness Gap} \\
\midrule
Training Data & 78.0\% & 23.3\% \\
JANUS Native & \textbf{100\%} & \textbf{0\%} \\
CTGAN + Rejection & 100\% & 0\%$^*$ \\
\bottomrule
\end{tabular}
\vspace{0.3em}
\footnotesize{Constraint: $\text{Salary}_{\text{offered}} \geq \text{Salary}_{\text{requested}}$. $^*$Approximately 15$\times$ slower due to rejection sampling overhead.}
\end{table}

The constraint $\text{Salary}_{\text{offered}} \geq \text{Salary}_{\text{requested}}$ ensures equal treatment \emph{within rows}---a form of individual fairness that statistical parity metrics cannot capture. JANUS handles this natively; baselines require rejection sampling.

\begin{figure}[htbp]
\centering
\includegraphics[width=0.65\columnwidth]{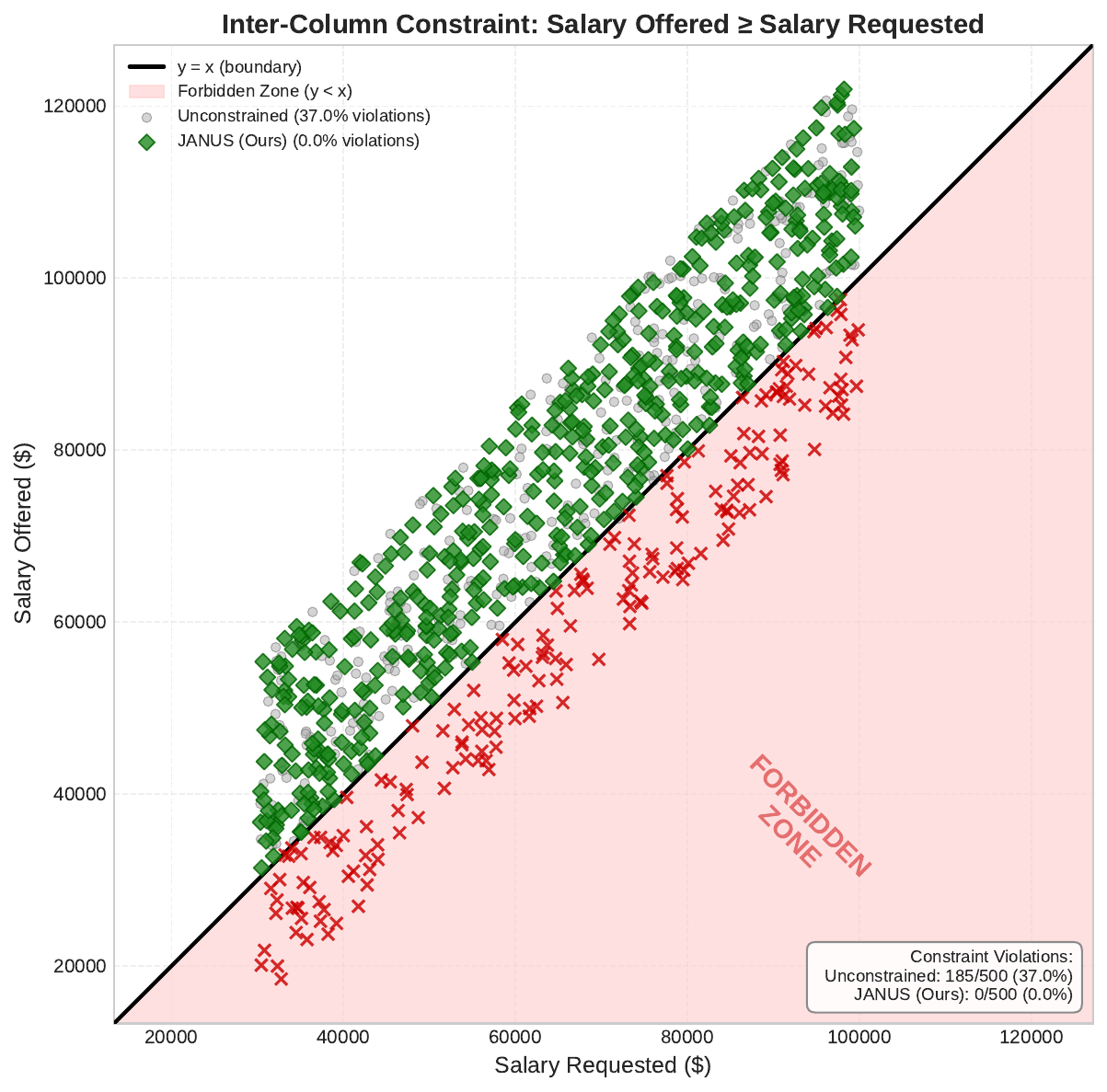}
\caption{\textbf{Forbidden Zone Analysis.} Scatter plot for $\text{Salary}_{\text{offered}} \geq \text{Salary}_{\text{requested}}$. The diagonal is the constraint boundary; the red region below is ``forbidden.'' Baselines (grey) violate the constraint; JANUS (blue) adheres perfectly with 0\% violations.}
\label{fig:forbidden_zone}
\end{figure}

Table~\ref{tab:fairness_capability} summarizes what fairness research JANUS uniquely enables.

\begin{table}[htbp]
\centering
\caption{Fairness Research Capability Comparison}
\label{tab:fairness_capability}
\begin{adjustbox}{max width=\columnwidth}
\begin{tabular}{@{}lcccc@{}}
\toprule
\textbf{Capability} & \textbf{Real} & \textbf{CTGAN} & \textbf{BN} & \textbf{JANUS} \\
\midrule
Known bias ground truth & \xmark & \xmark & \xmark & \cmark \\
Causal path control & \xmark & \xmark & \xmark & \cmark \\
Bias reversal & \xmark & \xmark & \xmark & \cmark \\
Inter-column fairness & \xmark & \xmark & \xmark & \cmark \\
Degenerate detection & \xmark & \xmark & \xmark & \cmark \\
Subgroup injection & \xmark & $\sim$ & $\sim$ & \cmark \\
\bottomrule
\end{tabular}
\end{adjustbox}
\end{table}

JANUS is the only framework enabling all six capabilities, providing comprehensive infrastructure for fairness algorithm development and evaluation.

\section{Conclusion}

JANUS breaks the \textbf{Trilemma} of synthetic data generation, achieving: \textbf{Fidelity} (Detection Score 0.497, best correlation preservation), \textbf{Control} (100\% constraint satisfaction via Reverse-Topological Back-filling, 49.6$\times$ speedup), and \textbf{Reliability} (analytical uncertainty decomposition, 128$\times$ faster than MC Dropout). The key technical insight is that the $\lambda_{\text{unsup}}$ term in our Hybrid Splitting Criterion forces trees to learn $P(X|Y)$ alongside $P(Y|X)$, enabling constraint propagation with $O(d)$ complexity instead of rejection sampling.

JANUS provides the first rigorous, efficient testbed for fairness auditing that doesn't rely on ``black box'' generation. Inter-column constraints ($X_i \geq X_j$) enable \emph{substantive} fairness enforcement---ensuring equal treatment within rows, not just equal rates across groups.

\textbf{Limitations.} Global discretization may lose precision for heavy-tailed distributions or features with extremely high cardinality. The back-filling algorithm handles multiple constrained children greedily; theoretical guarantees for complex constraint intersections require further development. JANUS's uncertainty estimates measure data-level uncertainty (leaf count statistics), not prediction-level uncertainty (model confidence), making them complementary to MC Dropout or ensembles.

\textbf{Future Work.} Promising directions include differentiable structure learning via NOTEARS~\cite{zheng_dags_2018}, continuous leaves using Gaussian Processes or normalizing flows~\cite{khemakhem_causal_2021}, active fairness learning using epistemic uncertainty~\cite{breugel_decaf_2021}, streaming updates for online data~\cite{martins_generation_2024}, and differential privacy mechanisms~\cite{aitsam_differential_2022,machanavajjhala_l-diversity_2007}.

All experiments use fixed random seeds and are reproducible with the provided scripts.

\printbibliography

\end{document}